\documentclass{article}

\PassOptionsToPackage{numbers, sort}{natbib}

\usepackage[preprint]{neurips_2026}

\makeatletter
\renewcommand{\@toptitlebar}{%
  \vskip 0.25in
  \vskip -\parskip%
}
\renewcommand{\@bottomtitlebar}{%
  \vskip 0.29in
  \vskip -\parskip
  \vskip 0.09in%
}
\renewcommand{\@notice}{}
\makeatother

\usepackage[utf8]{inputenc} 
\usepackage[T1]{fontenc}    
\usepackage[
        breaklinks,
        colorlinks,
        ]{hyperref}         
\usepackage{booktabs}       
\usepackage{amsfonts}       
\usepackage{amssymb}        
\usepackage{nicefrac}       
\usepackage{microtype}      
\usepackage{graphicx}
\usepackage{amsmath}
\usepackage{amsthm}
\usepackage[table,xcdraw,dvipsnames]{xcolor}
\usepackage{algorithm}
\usepackage{algorithmicx}
\usepackage{algpseudocode}
\usepackage{multirow}
\usepackage{subcaption} 
\usepackage{makecell}
\usepackage{wrapfig}
\usepackage{enumitem}
\usepackage[most]{tcolorbox}
\usepackage{eso-pic}
\usepackage{tikz}
\definecolor{c1}{RGB}{70,120,180}

\theoremstyle{plain}
\newtheorem{theorem}{Theorem}

\newtheorem{proposition}[theorem]{Proposition}
\theoremstyle{definition}

\theoremstyle{remark}

\usepackage{mathpazo}

\title{Self-Induced Outcome Potential: Turn-Level Credit Assignment for Agents without Verifiers}

\author{%
Senkang Hu$^{1,2,3,}\thanks{Part of this work was completed while the author was a visiting PhD student at the University of Oxford.}$\ \ , Yong Dai$^{4}$, Xudong Han$^{5}$, Zhengru Fang$^{1,2}$, \textbf{Yuzhi Zhao}$^6$,\\ \textbf{Sam Tak Wu Kwong}$^7$,  \textbf{Yuguang Fang}$^{1,2}$ \\
  $^1$Hong Kong JC STEM Lab of Smart City,
  $^2$City University of Hong Kong,\\
  $^3$University of Oxford,
  $^4$Fudan University,
  $^5$University of Sussex,\\
  $^6$Huazhong University of Science and Technology, $^7$Lingnan University\\
  \texttt{senkang.forest@my.cityu.edu.hk}
}

\begin{document}
\maketitle

\vspace{-1.5em}

\begin{abstract}
Long-horizon LLM agents succeed or fail through intermediate information-gathering turns, yet training signals are usually observed only at the final answer, because process-level rewards require extensive high-quality human annotation. Existing turn-level shaping methods address this supervision mismatch by rewarding turns that increase the likelihood of a gold answer, but they require answer supervision or stable task-specific verifiers that are often unavailable in agentic settings. Conversely, recent label-free reinforcement learning (RL) methods extract useful self-signals from output distributions, but mostly at the answer or trajectory level and therefore cannot attribute value to intermediate turns. To address these issues, we propose \emph{Self-Induced Outcome Potential} (SIOP), a framework that treats semantic clusters of final answers as latent future outcome states for potential-based turn-level credit assignment. For each query, we sample multiple rollouts, cluster their final answers into semantic outcome modes, and construct a reliability-aware target distribution over these latent states. We first formulate a distributional potential that rewards turns for increasing posterior support for reliable future outcome states, and then instantiate it with a tractable cluster-level approximation. Terminal outcome mass is added only at the final answer turn and propagated backward through turn-level advantages. The resulting objective generalizes information-potential shaping from gold-answer supervision to settings without task-specific gold verifiers while avoiding the broadcasted rollout-level advantages used by standard GRPO.
We formalize the framework, establish its reduction to supervised turn-level shaping in the gold-answer limit, and show that SIOP consistently improves over verifier-free outcome-level baselines on seven search-augmented agentic reasoning benchmarks while approaching the average performance of a gold-supervised outcome baseline. 
Code is available at \url{https://github.com/dl-m9/SIOP.git}.
\end{abstract}

\section{Introduction}

Recent reasoning models such as OpenAI o1 \cite{openai2024o1systemcard} and DeepSeek-R1 \cite{guoDeepSeekR1IncentivizesReasoning2025} have made a broader shift visible: large language models (LLMs) are increasingly trained to solve problems through extended reasoning rather than immediate answer generation. 
Agentic settings extend this paradigm from internal reasoning traces to external interaction. Instead of answering from a fixed prompt, an agent reasons over a partial state, retrieves evidence, calls tools, observes returned information, and decides whether further interaction is needed. This turns model behavior from static generation into sequential decision making over reasoning, tool use, and external observations. 

The shift is visible across recent systems and training recipes. Agentic retrieval-augmented generation (RAG) systems surveyed by Li \textit{et al.}\ \cite{li_towards_2025} frame retrieval as adaptive interaction rather than static preprocessing. Kimi K2 \cite{kimiteam2026kimik2openagentic} uses large-scale agentic data synthesis and joint reinforcement learning through real and synthetic environments. ReSearch \cite{chen2025researchlearningreasonsearch} and Search-R1 \cite{jin2025searchr} train search-augmented models to interleave reasoning and retrieval. ToolRL \cite{qian_toolrl_2025} and SimpleTIR \cite{xue2026simpletir} study tool-use RL and tool-integrated reasoning under multi-turn interaction. These works show that agentic RL has become a practical path for improving interactive reasoning.

However, many of these training recipes still rely on outcome-level feedback: the policy receives rewards mainly from final-answer correctness, task success, or terminal score after the full trajectory is complete. Such outcome-only reinforcement learning is not well suited to interactive agents. In long trajectories, a single terminal reward cannot reliably identify whether a specific thought, tool call, or observation improved the agent's future prospects. For example, ReasonRAG \cite{zhang2025processvsoutcomereward} and InForage \cite{qian2025scentknowledgeoptimizingsearchenhanced} report that outcome-supervised agentic RL suffers from low exploration efficiency, gradient conflict, and sparse rewards. These studies indicate that terminal outcome supervision is not enough to guide the agent's long-horizon behavior.

To address this problem, one line of work proposes turn-level reward shaping. For example, methods such as IGPO \cite{wang2025igpo}, $\Delta$Belief-RL \cite{auzina2026intrinsiccreditassignmentlong}, TIPS \cite{xie2026tips}, and InfoPO \cite{kong2026infopo} measure whether an interaction turn increases the model's support for a desired target, such as the gold answer or the correct next action. These methods show that information-potential shaping can substantially improve multi-turn optimization.
However, these methods inherit a strong assumption: a rewardable target must already be available. In many agentic tasks, there is no unique gold answer or stable task-specific verifier against which a turn-level potential can be computed.

Another line of work removes this requirement by extracting training signals directly from model outputs. EMPO \cite{zhang2025right} clusters semantically equivalent responses and optimizes semantic entropy without external supervision, while TTRL \cite{zuo_ttrl_2025} performs test-time reinforcement learning by selecting the majority answer from repeated samples as a self-generated pseudo-target. Related verifier-free RL methods, including RLPR \cite{yu_rlpr_2025}, NOVER \cite{liu_nover_2025}, and DRO \cite{xu_direct_2025}, similarly show that model-internal confidence or answer distributions can support learning beyond tasks with explicit verifiers. Yet these methods remain mostly outcome-level: they tell the model which final responses appear more reliable, but not which intermediate turns actually moved the trajectory toward those outcomes.

Building on these gaps, this paper studies how to lift unlabeled outcome-level reliability signals into turn-level credit assignment. Our key abstraction is to cluster sampled final answers into \emph{semantic outcome modes}, where each mode is a query-specific group of answers that express the same underlying meaning, and then treat these modes as \emph{latent future states}. Even without gold answers or stable task-specific verifiers, a batch of rollouts for the same query induces a structured distribution over these latent states. Some modes are frequent, consistent, evidence-supported, or otherwise more trustworthy than others, while others reflect low-consensus noise and should be down-weighted. We then define a potential function over this self-induced latent state space and reward each turn by how much it increases the model's support for reliable future outcome states. The resulting \emph{self-induced outcome potential} generalizes information-potential shaping from gold-answer supervision to verifier-free settings without task-specific gold verifiers.

Concretely, we sample multiple full rollouts for each query, cluster their final answers into semantic outcome modes, and assign soft reliability weights to these modes. This yields a distributional target over reliable latent futures rather than a hard pseudo-label. Because estimating the full mode distribution at every turn is expensive, we instantiate this principle in practice with a multi-reference cluster-level approximation that conditions each rollout on its realized semantic outcome mode while scoring a small set of reference answers rather than a single prototype. Terminal outcome mass is added only at the final answer turn and propagated backward through turn-level advantages, so different turns of the same rollout receive different credit. This design preserves the key intuition behind potential-based shaping, namely rewarding the intermediate interactions that move the trajectory toward reliable future states, while avoiding the assumption that a gold answer is known during training.
Our contributions are summarized as follows.
\begin{enumerate}
    \item We formulate \emph{verifier-free turn-level credit assignment} as potential-based learning over self-induced semantic outcome modes, where clustered final answers serve as latent future states rather than fixed gold targets or trajectory-level pseudo-labels.
    \item We define a \emph{self-induced outcome potential} (SIOP) over these latent states, calibrating empirical outcome modes by reliability signals and converting prefix-to-prefix changes in cluster support into turn-level process rewards.
    \item We construct turn-specific advantages for long-horizon agents, assigning credit only to trainable tokens within each turn rather than broadcasting a trajectory-level score.
    \item We provide theoretical characterizations showing how the proposed reward reduces to gold-answer shaping in the supervised limit and why it measures step-wise marginal support gain rather than broadcasting a trajectory score. We also present an empirical study on seven search-QA benchmarks showing that our method outperforms outcome-only verifier-free baselines and approaches a gold-supervised outcome RL baseline.
\end{enumerate}

\section{Problem Setup}
\label{sec:problem}

We study an agentic setting where an agent receives a query $q$ and interacts with an environment over multiple turns before producing a final answer. At turn $t$, the agent emits reasoning tokens, optionally issues a tool call, and receives an observation. We treat this as a single interaction segment
\begin{equation}
s_t = (\text{thought}_t, \text{tool}_t, \text{obs}_t),
\end{equation}
so that a complete rollout can be written directly as
\begin{equation}
\tau = (q, s_1, \dots, s_T, y).
\end{equation}
For convenience, we write $\tau_{\le t} = (q, s_1, \dots, s_t)$ for the rollout truncated after turn $t$.
For each query, we sample $K$ rollouts from the current policy $\pi_\theta$:
\begin{equation}
\tau^{(k)} \sim \pi_\theta(\cdot \mid q), \qquad k=1,\dots,K,
\end{equation}
where the $k$-th rollout can be written as
\begin{equation}
\tau^{(k)} = (q, s_1^{(k)}, \dots, s_{T_k}^{(k)}, y^{(k)}).
\end{equation}

\textbf{Supervision mismatch.}
Supervised turn-level shaping methods such as $\Delta$Belief-RL \cite{auzina2026intrinsiccreditassignmentlong}, IGPO \cite{wang2025igpo}, and TIPS \cite{xie2026tips} assume access to a gold answer $y^\star$ or an equivalent target representation. In our setting, training does not observe a unique correct answer, a stable answer set, or a task-specific correctness verifier. The task-specific information available during learning consists of sampled rollouts from the current policy.

\textbf{Verifier-free terminology.}
We use \emph{verifier-free} in the task-specific sense: training requires no gold answers, task-specific correctness verifiers, or hand-designed rewards. An off-the-shelf natural language inference (NLI) model is used only as a task-agnostic semantic primitive for clustering and evidence scoring, without access to benchmark labels or task-specific answer checkers.

\textbf{Semantic outcome modes.}
Let $\mathcal{Y}_q = \{y^{(1)}, \dots, y^{(K)}\}$ denote the set of answers sampled for query $q$. Following semantic entropy by Kuhn et al.\ \cite{kuhn2023semantic}, we cluster these answers by bidirectional semantic entailment under the query context. Writing $y \Rightarrow y'$ for ``$y$ semantically entails $y'$ given $q$,'' as judged by an NLI model applied to the query-contextualized pair, we assign $y$ and $y'$ to the same cluster iff $y \Rightarrow y'$ and $y' \Rightarrow y$. Applying this criterion to $\mathcal{Y}_q$ yields a partition
\begin{equation}
\mathcal{C}(q)=\{c_1,\dots,c_M\},
\end{equation}
where each cluster $c_m \subseteq \mathcal{Y}_q$ collects answers with the same meaning. We call each cluster a \emph{semantic outcome mode}. We write $c(k) \in \mathcal{C}(q)$ for the cluster containing the final answer $y^{(k)}$ of rollout $k$.

\textbf{Cluster statistics.}
For each cluster $c \in \mathcal{C}(q)$, let
\begin{equation}
S_c = \{k \in \{1,\dots,K\} \mid c(k)=c\}
\end{equation}
denote the set of sampled rollouts whose final answers fall into that semantic outcome mode. The corresponding empirical cluster mass is
\begin{equation}
m(c \mid q)=\frac{|S_c|}{K}.
\end{equation}
We use $q_\theta(c \mid q)$ to denote the training target distribution over semantic outcome modes. Section~\ref{sec:method} instantiates this distribution by calibrating the empirical cluster mass with reliability signals.

\textbf{Goal.}
Our objective is to construct a turn-level reward $r^{(k)}_{t,\mathrm{proc}}$ for each segment $s_t^{(k)}$ and convert these rewards into turn-specific learning signals for policy optimization. The next section defines the self-induced target distribution, outcome potential, and policy update used to compute these quantities.

\section{Methodology}
\label{sec:method}
\vspace{-.5em}

SIOP starts from a credit-assignment question: without a gold answer or task-specific verifier, what future state should an intermediate turn be rewarded for moving toward? We answer this by using semantic outcome modes as the future-state space for potential-based shaping. Figure~\ref{fig:siop_overview} shows the resulting pipeline. Appendix~\ref{app:notation} summarizes the notation used throughout this section.

\begin{figure}[t]
    \centering
    \includegraphics[width=\linewidth]{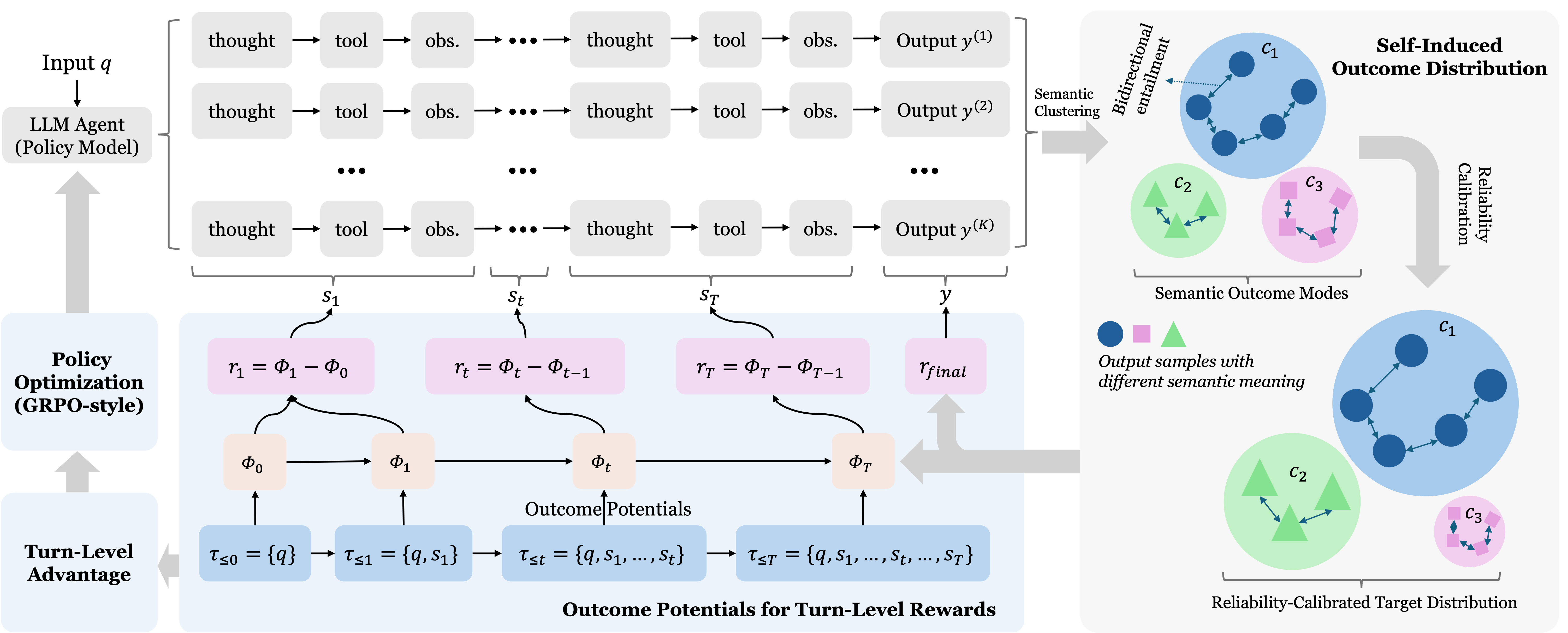}
    \caption{Overview of Self-Induced Outcome Potential (SIOP). For each query, the policy samples multiple agent rollouts, clusters their final answers into semantic outcome modes, and calibrates these modes into a reliability-weighted target distribution. SIOP then evaluates each prefix by its outcome potential, assigns turn rewards from potential differences, adds terminal outcome mass at the final answer turn, and updates the policy with turn-level advantages.}
    \label{fig:siop_overview}
    \vspace{-1em}
\end{figure}

\subsection{Self-Induced Outcome Distribution}

\noindent\textbf{Semantic outcome modes.}
Multiple rollouts for the same query often produce answers that differ in surface form but express the same semantic hypothesis. Treating each string as a separate target makes the reward sensitive to wording, while selecting only the majority answer collapses uncertainty too early. We therefore group final answers into semantic outcome modes and treat each mode as a latent future state.

Let $S_c = \{k \mid c(k)=c\}$ denote the rollout subset assigned to cluster $c$. We begin with the empirical semantic outcome distribution
\begin{equation}
m(c \mid q)=\frac{|S_c|}{K},
\end{equation}
and let $R(c) \subseteq c$ denote a small reference set selected from cluster $c$, with nonnegative weights $\{\alpha_y^{(c)}\}_{y \in R(c)}$ satisfying
\begin{equation}
\sum_{y \in R(c)} \alpha_y^{(c)} = 1.
\end{equation}
These reference answers provide a tractable semantic proxy for the full cluster while avoiding commitment to a single canonical wording.

\noindent\textbf{Reliability calibration.}
Empirical mode frequency is useful, but it is not a correctness guarantee. A wrong answer can become the majority when the base policy has a consistent bias or when retrieval returns misleading evidence. SIOP therefore does not directly optimize raw cluster frequency. It reweights the rollout-induced distribution using reliability evidence, so a frequent mode receives strong target mass only when it is also supported by the available semantic or evidential signal.

We then calibrate the empirical distribution with a reliability correction
\begin{equation}
u_\theta(c,q) = \eta \, r(c,q),
\end{equation}
where $\eta \geq 0$ controls the correction strength and $r(c,q)$ is a domain-appropriate reliability signal. The framework only requires that larger $r(c,q)$ indicate more reliable outcome modes.

For tool-using agents, we instantiate $r$ as \emph{evidence support}, namely the average degree to which tool observations along trajectories support the cluster's reference set:
\begin{equation}
r(c,q) = \frac{1}{|S_c|}\sum_{k \in S_c}
\frac{1}{T_k}\sum_{t=1}^{T_k}
\sum_{y \in R(c)} \alpha_y^{(c)}
s_{\mathrm{NLI}}\!\bigl(\mathrm{obs}_t^{(k)},\, y\bigr),
\end{equation}
where $s_{\mathrm{NLI}}(\mathrm{obs}, y)$ is the entailment confidence from an NLI model.
In reasoning-only settings without tool observations, $r(c,q)$ can instead be instantiated with cluster-level self-consistency or mean log-confidence over the same reference set.

The resulting self-induced target distribution is
\begin{equation}
q_\theta(c \mid q)=\frac{m(c \mid q)\exp(u_\theta(c,q))}{\sum_{c' \in \mathcal{C}(q)} m(c' \mid q)\exp(u_\theta(c',q))}.
\end{equation}
If $u_\theta(c,q)=0$ for all clusters, this reduces to the empirical semantic outcome distribution induced by the sampled rollouts. Hence the self-induced target distribution can be viewed as the model's own semantic answer distribution, reweighted by reliability signals before it is used for turn-level shaping.

\subsection{Semantic Outcome Potential for Turn-Level Rewards}

\noindent\textbf{Multi-reference cluster support.}
After forming reliable outcome modes, turn credit requires a prefix-conditioned support score for each mode. Because agentic answer spaces make exact mode probabilities impractical, we approximate mode support using a small reference set selected from the cluster.

The framework evaluates a turn against the distribution of reliable future outcome modes. For a truncated rollout $\tau_{\le t}$, the true posterior mass on semantic outcome mode $c$ is
\begin{equation}
p_\theta(c \mid \tau_{\le t})
= \sum_{y \in \mathcal{Y}(q)} 
p_\theta(y \mid \tau_{\le t})\,\mathbf{1}[y \in c].
\end{equation}
Because summing over the full answer space is intractable, we instantiate a tractable multi-reference surrogate for cluster support:
\begin{equation}
\tilde p_\theta(c \mid \tau_{\le t})
= \sum_{y \in R(c)} \alpha_y^{(c)} p_\theta(y \mid \tau_{\le t}),
\end{equation}
which scores a small set of reference answers rather than the full cluster. When $R(c)$ contains a single answer, this recovers the earlier single-prototype approximation as a special case. With multiple references, however, the reward is less sensitive to any one wording choice and better reflects support for the semantic mode as a set.

\textbf{Potential-based turn reward.}
Once reliable semantic modes and prefix-conditioned support are defined, SIOP assigns turn-level credit by measuring changes in a potential over the rollout's realized semantic outcome mode $c(k)$. A turn receives positive credit when it increases prefix support for this mode and negative credit when it decreases that support. The \emph{outcome potential} at turn $t$ is
\begin{equation}
\Phi_\theta(\tau_{\le t}^{(k)}, c(k))
= q_\theta(c(k) \mid q)\,
\log\bigl(\tilde p_\theta(c(k) \mid \tau_{\le t}^{(k)})\bigr),
\end{equation}
The potential combines a reliability-calibrated target mass with prefix-conditioned support. The target mass $q_\theta(c(k)\mid q)$ reflects the rollout frequency of the realized cluster after calibration by evidence support, while the log-support term measures how strongly the current prefix predicts that cluster. Thus potential differences are largest when a turn increases support for a cluster with high calibrated target mass.
The step-level process reward is:
\begin{equation}
\begin{aligned}
r^{(k)}_{t,\mathrm{proc}} &=
\Phi_\theta(\tau_{\le t}^{(k)}, c(k))
- \Phi_\theta(\tau_{\le t-1}^{(k)}, c(k)) \\
&= q_\theta(c(k) \mid q)\Bigl[
\log\bigl(\tilde p_\theta(c(k) \mid \tau_{\le t}^{(k)})\bigr)
- \log\bigl(\tilde p_\theta(c(k) \mid \tau_{\le t-1}^{(k)})\bigr)
\Bigr].
\end{aligned}
\end{equation}
This reward is positive when turn $t$ increases conditional support for the rollout's semantic outcome mode, and negative when it moves the trajectory away from that mode. All rollouts assigned to the same cluster share the same semantic reference set $R(c)$, so the reward operates at the cluster level rather than the individual answer level.

\begin{proposition}[Telescoping of the outcome potential]
\label{prop:telescoping}
For any rollout $\tau^{(k)}$, let $\tau_{\le 0}^{(k)}$ denote the initial query state before agent-generated turns. The cumulative process reward satisfies
\vspace{-0.5em}
\begin{equation}
\begin{aligned}
\sum_{t=1}^{T_k} r^{(k)}_{t,\mathrm{proc}}
&=
\Phi_\theta(\tau_{\le T_k}^{(k)}, c(k))
-\Phi_\theta(\tau_{\le 0}^{(k)}, c(k)) \\
&=
q_\theta(c(k) \mid q)\Bigl[
\log\bigl(\tilde p_\theta(c(k) \mid \tau_{\le T_k}^{(k)})\bigr)
- \log\bigl(\tilde p_\theta(c(k) \mid q)\bigr)
\Bigr].
\end{aligned}
\end{equation}
Hence the process reward is potential-based: cumulative credit depends only on the change in outcome potential between the initial query state and the final trajectory state.
\end{proposition}
\vspace{-0.5em}

This telescoping form makes the process reward genuinely step-wise. The target distribution $q_\theta(c(k) \mid q)$ determines which outcome modes receive calibrated target mass, while the potential difference determines which turns increased conditional support for the cluster reference set.

\begin{proposition}[Factorization of outcome-potential shaping]
\label{prop:reduction}
The per-turn SIOP reward induced by $\Phi_\theta$ factors as
\vspace{-0.5em}
\begin{equation}
r^{(k)}_{t,\mathrm{proc}}
= \underbrace{q_\theta(c(k) \mid q)}_{\text{calibrated target mass}}
\cdot
\underbrace{\Bigl[\log\!\bigl(\tilde p_\theta(c(k) \mid \tau_{\le t}^{(k)})\bigr)
- \log\!\bigl(\tilde p_\theta(c(k) \mid \tau_{\le t-1}^{(k)})\bigr)\Bigr]}_{\text{cluster-level support gain}},
\end{equation}
i.e., cluster-level support shaping weighted by the calibrated target mass of that cluster. In the supervised limit where $q_\theta(c^\star \mid q)=1$, $R(c^\star)=\{y^\star\}$, and $\alpha_{y^\star}^{(c^\star)}=1$, the reward recovers standard gold-answer support shaping:
\begin{equation}
r^{(k)}_{t,\mathrm{proc}} = \log p_\theta(y^\star \mid \tau_{\le t}^{(k)}) - \log p_\theta(y^\star \mid \tau_{\le t-1}^{(k)}).
\end{equation}
\end{proposition}

Proposition~\ref{prop:reduction} reframes SIOP as calibrated target-mass weighting of support shaping: $q_\theta(\cdot \mid q)$ is a query-specific target distribution over semantic modes, while the log-difference term measures turn-wise cluster support gain. Gold-answer shaping is recovered when the target distribution concentrates on the gold cluster. Diffuse or low-mass target distributions attenuate credit. Appendix~\ref{app:bayesian_calibration} gives a Bayesian reading of the calibration step and proves when reliability evidence can reverse a spurious majority.

\subsection{Turn-Level Advantage Construction}

Process rewards identify which turns change semantic support, but global answer quality is still observed at the trajectory endpoint. We encode these two signals in an augmented immediate reward: process credit is assigned at every turn, while terminal outcome mass is added only at the final answer turn. For rollout $k$, this reward is
\begin{equation}
\bar r_t^{(k)} =
\lambda r^{(k)}_{t,\mathrm{proc}}
+ \mathbf{1}[t=T_k](1-\lambda)r_{\mathrm{final}}^{(k)},
\end{equation}
where $r_{\mathrm{final}}^{(k)} = q_\theta(c(k) \mid q)$ is the target-distribution weight of the rollout's semantic outcome mode. This design gives the final answer turn a direct outcome signal while allowing earlier turns to benefit only through return propagation.

For each query group (the $K$ rollouts of one prompt), we pool all augmented rewards across its rollouts and turns and compute group-normalized immediate rewards using the group's own statistics:
\begin{equation}
\tilde r_t^{(k)} = \frac{\bar r_t^{(k)} - \mu_g}{\sigma_g},\quad \text{where} \
\mu_g = \frac{1}{N_g}\sum_{k,t}\bar r_t^{(k)} \ \text{and} \
\sigma_g = \sqrt{\frac{1}{N_g}\sum_{k,t}\bigl(\bar r_t^{(k)}-\mu_g\bigr)^2},
\end{equation}
with $N_g = \sum_{k=1}^K T_k$, the total number of (rollout, turn) pairs in that query group. Normalization is computed per query group (GRPO-style) rather than batch-globally, so different queries do not interfere with each other's turn scales. Pooling all turns within a group also means the $z$-score does not artificially rescale each turn index to the same variance: turns with systematically larger rewards retain their relative magnitudes within the group.

To propagate this terminal outcome signal back to earlier turns, we define the turn advantage as a discounted return over normalized immediate rewards:
\begin{equation}
A_t^{(k)} = \sum_{u=t}^{T_k} \gamma^{u-t} \tilde r_u^{(k)}.
\end{equation}
This construction resolves the core mismatch with broadcasted rollout-level advantages: the final reward is not copied to all tokens, but only affects earlier turns through the discounted return of turn-local normalized rewards.

\vspace{-0.5em}
\subsection{Turn-Conditioned Policy Optimization}
\vspace{-0.5em}

To use turn-level advantages in a token-level clipped objective, we first identify the policy tokens generated by each interaction turn.
Let $\mathcal T_t^{(k)}$ denote the set of trainable policy tokens generated within turn $t$ of rollout $k$, excluding prompts, previous context replay, tool observations, and other non-policy text. We assign the same turn advantage $A_t^{(k)}$ to all trainable tokens in $\mathcal T_t^{(k)}$, but not across different turns.

With token-level importance ratio $\rho_{t,j}^{(k)}$ for token $j \in \mathcal T_t^{(k)}$, we minimize a GRPO-style clipped objective:
\vspace{-0.5em}
\begin{equation}
\begin{aligned}
\mathcal L_{\mathrm{SIOP}}
=\;& - \mathbb E\!\Bigg[
\sum_{k=1}^K \sum_{t=1}^{T_k} \sum_{j \in \mathcal T_t^{(k)}}
\min\!\Big(
\rho_{t,j}^{(k)} A_t^{(k)},
\mathrm{clip}(\rho_{t,j}^{(k)}, 1-\epsilon_{\mathrm{clip}}, 1+\epsilon_{\mathrm{clip}})
A_t^{(k)}
\Big)
\Bigg] \\
&\quad + \beta \, \mathrm{KL}(\pi_\theta \,\|\, \pi_{\theta_{\mathrm{old}}}).
\end{aligned}
\end{equation}
We refer to this objective as the \emph{SIOP objective}. It preserves the main appeal of group-based RL, namely critic-free optimization with normalized on-policy signals, while ensuring that the optimization signal remains turn-specific. Within a turn, tokens share the same local credit, but advantages differ across turns. Full training pseudocode is deferred to Appendix~\ref{app:algorithm}.

\vspace{-0.5em}
\section{Experiments}
\vspace{-0.5em}
\label{sec:experiment}

\begin{table}[t]
    \centering
    \small
    \caption{Main comparative results, reported as Exact Match and token-level F1 for 4B and 8B backbones. Gold-supervised baselines are shown in gray and are not bolded. TTRL, EMPO, and SIOP use no task-specific verifier. \textbf{Bold} marks the best verifier-free result within each metric block.}
    \resizebox{\linewidth}{!}{
    \begin{tabular}{p{3.7cm}|ccccccc|c}
    \toprule
    \multirow{2.4}{*}{Method} & \multicolumn{3}{c}{Single-Hop QA} & \multicolumn{4}{c|}{Multi-Hop QA} & \multirow{2.4}{*}{\textbf{Avg.}} \\
    \cmidrule(lr){2-4}\cmidrule(lr){5-8}
     & NQ & TriviaQA & PopQA & HotpotQA & 2Wiki & MuSiQue & Bamboogle & \\
    \midrule
    \emph{Exact Match Accuracy (4B)} &&&&&&&& \\
    \textcolor{gray!80}{GRPO (gold-sup.)} & \textcolor{gray!80}{0.283} & \textcolor{gray!80}{0.623} & \textcolor{gray!80}{0.436} & \textcolor{gray!80}{0.330} & \textcolor{gray!80}{0.340} & \textcolor{gray!80}{0.121} & \textcolor{gray!80}{0.464} & \textcolor{gray!80}{0.371} \\
    \textcolor{gray!80}{IGPO (gold-sup.)} & \textcolor{gray!80}{0.262} & \textcolor{gray!80}{0.598} & \textcolor{gray!80}{0.449} & \textcolor{gray!80}{0.359} & \textcolor{gray!80}{0.340} & \textcolor{gray!80}{0.113} & \textcolor{gray!80}{0.464} & \textcolor{gray!80}{0.369} \\
    TTRL    & 0.262 & 0.607 & 0.416 & 0.318 & 0.314 & 0.090 & 0.400 & 0.344 \\
    EMPO    & 0.270 & \textbf{0.611} & 0.412 & 0.330 & 0.301 & 0.090 & \textbf{0.456} & 0.353 \\
    SIOP & \textbf{0.281} & 0.604 & \textbf{0.416} & \textbf{0.344} & \textbf{0.367} & \textbf{0.100} & 0.408 & \textbf{0.360} \\
    \midrule
    \emph{Token-level F1 (4B)} &&&&&&&& \\
    \textcolor{gray!80}{GRPO (gold-sup.)} & \textcolor{gray!80}{0.408} & \textcolor{gray!80}{0.703} & \textcolor{gray!80}{0.478} & \textcolor{gray!80}{0.441} & \textcolor{gray!80}{0.409} & \textcolor{gray!80}{0.183} & \textcolor{gray!80}{0.534} & \textcolor{gray!80}{0.451} \\
    \textcolor{gray!80}{IGPO (gold-sup.)} & \textcolor{gray!80}{0.396} & \textcolor{gray!80}{0.682} & \textcolor{gray!80}{0.486} & \textcolor{gray!80}{0.451} & \textcolor{gray!80}{0.381} & \textcolor{gray!80}{0.184} & \textcolor{gray!80}{0.550} & \textcolor{gray!80}{0.447} \\
    TTRL    & 0.402 & 0.691 & 0.458 & 0.401 & 0.357 & 0.157 & 0.520 & 0.427 \\
    EMPO    & \textbf{0.404} & \textbf{0.696} & {0.460} & 0.414 & 0.347 & {0.159} & 0.510 & 0.427 \\
    SIOP & 0.400 & 0.691 & \textbf{0.465} & \textbf{0.443} & \textbf{0.411} & \textbf{0.179} & \textbf{0.528} & \textbf{0.445} \\
    \midrule[\heavyrulewidth]
    \emph{Exact Match Accuracy (8B)} &&&&&&&& \\
    \textcolor{gray!80}{GRPO (gold-sup.)} & \textcolor{gray!80}{0.270} & \textcolor{gray!80}{0.643} & \textcolor{gray!80}{0.418} & \textcolor{gray!80}{0.391} & \textcolor{gray!80}{0.363} & \textcolor{gray!80}{0.141} & \textcolor{gray!80}{0.488} & \textcolor{gray!80}{0.388} \\
    \textcolor{gray!80}{IGPO (gold-sup.)} & \textcolor{gray!80}{0.275} & \textcolor{gray!80}{0.641} & \textcolor{gray!80}{0.461} & \textcolor{gray!80}{0.402} & \textcolor{gray!80}{0.393} & \textcolor{gray!80}{0.178} & \textcolor{gray!80}{0.456} & \textcolor{gray!80}{0.401} \\
    TTRL    & 0.266 & 0.623 & 0.416 & 0.385 & 0.357 & \textbf{0.133} & 0.440 & 0.374 \\
    EMPO    & 0.251 & 0.630 & 0.415 & \textbf{0.392} & 0.357 & 0.120 & 0.472 & 0.377 \\
    SIOP & \textbf{0.270} & \textbf{0.646} & \textbf{0.445} & 0.348 & \textbf{0.412} & 0.107 & \textbf{0.488} & \textbf{0.388} \\
    \midrule
    \emph{Token-level F1 (8B)} &&&&&&&& \\
    \textcolor{gray!80}{GRPO (gold-sup.)} & \textcolor{gray!80}{0.422} & \textcolor{gray!80}{0.728} & \textcolor{gray!80}{0.474} & \textcolor{gray!80}{0.484} & \textcolor{gray!80}{0.416} & \textcolor{gray!80}{0.224} & \textcolor{gray!80}{0.599} & \textcolor{gray!80}{0.478} \\
    \textcolor{gray!80}{IGPO (gold-sup.)} & \textcolor{gray!80}{0.407} & \textcolor{gray!80}{0.730} & \textcolor{gray!80}{0.498} & \textcolor{gray!80}{0.499} & \textcolor{gray!80}{0.461} & \textcolor{gray!80}{0.276} & \textcolor{gray!80}{0.583} & \textcolor{gray!80}{0.493} \\
    TTRL    & 0.401 & 0.722 & 0.476 & 0.431 & 0.416 & 0.192 & 0.539 & 0.454 \\
    EMPO    & \textbf{0.406} & 0.717 & 0.450 & 0.445 & 0.420 & 0.193 & 0.580 & 0.459 \\
    SIOP & 0.403 & \textbf{0.736} & \textbf{0.479} & \textbf{0.453} & \textbf{0.474} & \textbf{0.203} & \textbf{0.612} & \textbf{0.480} \\
    
    \bottomrule
    \end{tabular}
    }
    \label{tab:prelim_main}
    \vspace{-1em}
    \end{table}

\begin{figure}[t]
    \centering
    \includegraphics[width=\linewidth]{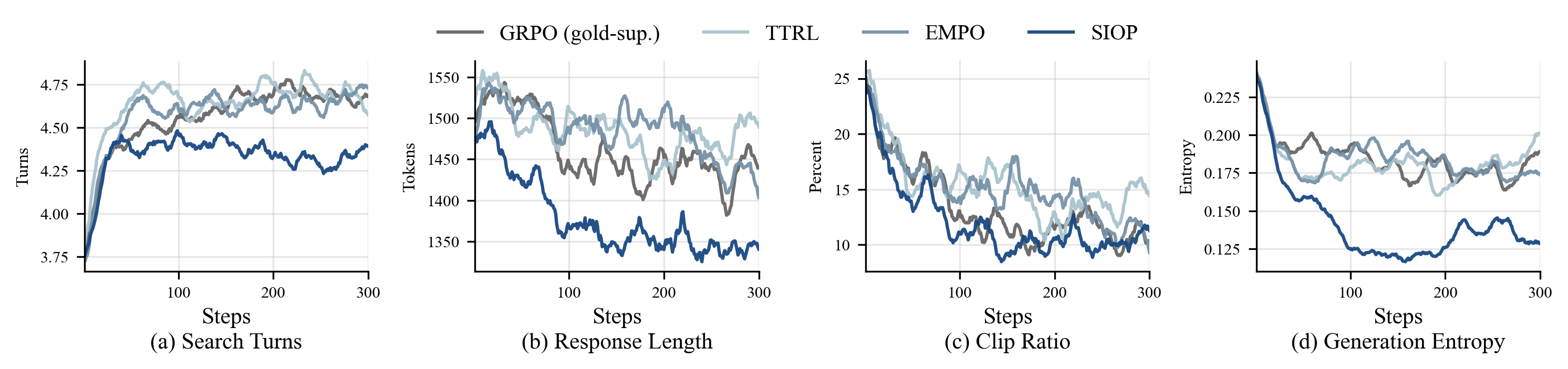}
    \vspace{-1.5em}
    \caption{Training dynamics, including search turns, response length, response-length clip ratio, and generation entropy.}
    \label{fig:behavior_training_curves}
    \vspace{-1.5em}
    \end{figure}

\begin{table}[t]
    \centering
    \small
    \caption{Component ablation. Each row removes one ingredient from \emph{Full SIOP}: multi-reference cluster support, reliability calibration, both together, or the per-turn process reward. The agent loop, retriever, optimizer, and turn-level advantage construction are kept identical across rows, so each row isolates a single design choice. \textbf{Bold} marks the best entry in each column.}
    \resizebox{\linewidth}{!}{
    \begin{tabular}{p{4cm}|ccccccc|c}
    \toprule
    \multirow{2.4}{*}{Method} & \multicolumn{3}{c}{Single-Hop QA} & \multicolumn{4}{c|}{Multi-Hop QA} & \multirow{2.4}{*}{\textbf{Avg.}} \\
    \cmidrule(lr){2-4}\cmidrule(lr){5-8}
     & NQ & TriviaQA & PopQA & HotpotQA & 2Wiki & MuSiQue & Bamboogle & \\
    \midrule
    \emph{Accuracy (Exact Match)} &&&&&&&& \\
    Full SIOP                        & \textbf{0.281} & 0.604          & 0.416          & \textbf{0.344} & \textbf{0.367} & \textbf{0.100} & \textbf{0.408}          & \textbf{0.360} \\
    \;w/o multi-ref                  & 0.266          & 0.604          & 0.408          & 0.318          & 0.328          & 0.090          & 0.384          & 0.343 \\
    \;w/o calibration                & 0.252          & 0.609          & 0.416          & 0.332          & 0.336          & 0.082          & 0.400          & 0.347 \\
    \;w/o multi-ref, w/o calib       & 0.256          & 0.596          & 0.418          & 0.311          & 0.313          & 0.076          & 0.407          & 0.340 \\
    \;w/o process reward             & 0.273          & \textbf{0.627} & \textbf{0.426} & 0.330          & 0.307          & 0.094 & 0.368          & 0.346 \\
    \midrule
    \emph{Token-level F1} &&&&&&&& \\
    Full SIOP                        & 0.400          & 0.691          & 0.465          & \textbf{0.443} & \textbf{0.411} & \textbf{0.179} & \textbf{0.528} & \textbf{0.445} \\
    \;w/o multi-ref                  & \textbf{0.407} & 0.699          & 0.463          & 0.405          & 0.379          & 0.155          & 0.481          & 0.427 \\
    \;w/o calibration                & 0.396          & 0.687          & \textbf{0.473} & 0.418          & 0.378          & 0.147          & 0.521          & 0.431 \\
    \;w/o multi-ref, w/o calib       & 0.403          & 0.687          & 0.459          & 0.401          & 0.354          & 0.144          & 0.522          & 0.424 \\
    \;w/o process reward             & 0.400          & \textbf{0.703} & 0.464          & 0.429          & 0.357          & 0.172          & 0.481          & 0.429 \\
    \bottomrule
    \end{tabular}
    }
    \label{tab:abl_components}
    \vspace{-1.5em}
    \end{table}

\subsection{Experimental Settings}

\textbf{Tasks and benchmarks.}
Our primary evaluation setting is search-augmented multi-turn question answering, which is the cleanest testbed for the proposed method because intermediate retrieval decisions are essential. Following this line of work, we train on a merged search-QA environment built from the Natural Questions (NQ) \cite{kwiatkowski-etal-2019-natural}, 2WikiMultiHopQA \cite{xanh2020_2wikimultihop}, and TriviaQA \cite{joshi-etal-2017-triviaqa} training splits and evaluate on seven benchmarks: NQ, TriviaQA \cite{joshi-etal-2017-triviaqa}, HotpotQA \cite{yang-etal-2018-hotpotqa}, 2WikiMultiHopQA \cite{xanh2020_2wikimultihop}, MuSiQue \cite{trivedi-etal-2022-musique}, Bamboogle \cite{press-etal-2023-measuring}, and PopQA \cite{mallen-etal-2023-trust}. We report results grouped by single-hop QA and multi-hop QA, matching the structure of Table~\ref{tab:prelim_main}.

\textbf{Backbone and environment.}
We use Qwen3-4B \cite{yang2025qwen3technicalreport} and Qwen3-8B \cite{yang2025qwen3technicalreport} as the policy backbone and run a search-augmented agent with a fixed interface across methods: at each turn the model produces reasoning text, optionally issues a search query, and receives retrieved passages before deciding the next turn or final answer. The prompting template, retrieval backend, passages per turn, and maximum interaction budget are held fixed across baselines so comparisons isolate the reward design. Full hyperparameters are listed in Appendix~\ref{app:setup}.

\textbf{Metrics.}
We report Exact Match (EM) and token-level F1 on the final extracted answer. In addition, we track behavior-level diagnostics in Figure~\ref{fig:behavior_training_curves}: average number of search turns, response length, response-length clip rate, and generation entropy.

\textbf{Baselines.}
SIOP is compared against two verifier-free baselines: {TTRL}~\cite{zuo_ttrl_2025}, which rewards only the dominant semantic cluster as a hard majority pseudo-target, and {EMPO}~\cite{zhang2025right}, which rewards each rollout by the empirical frequency of its answer's cluster. We additionally include two gold-supervised baselines for context: {GRPO}~\cite{shao_deepseekmath_2024}, which uses exact-match correctness of the gold answer as an outcome reward, and {IGPO}~\cite{wang2025igpo}, which rewards each turn according to how much it increases the model's probability of the gold answer.

\subsection{Comparative Results}
\label{sec:prelim_results}

Table~\ref{tab:prelim_main} reports EM and F1 on the seven-benchmark search-QA suite.
On the 4B backbone, SIOP is the best verifier-free method on average, improving over EMPO by +0.7\% EM and +1.8\% F1 and over TTRL by +1.6\% EM and +1.8\% F1. The gains are clearest on multi-hop benchmarks where intermediate retrieval matters: on HotpotQA and 2Wiki, SIOP beats both verifier-free baselines and also exceeds the gold-supervised GRPO outcome baseline on both metrics, with the 2Wiki EM gap to EMPO reaching +6.6\%. SIOP also closes most of the supervision gap to the gold-supervised IGPO turn-level baseline (-0.9\% EM, -0.2\% F1 on average) without using any gold labels at training time.
On the 8B backbone, SIOP again achieves the strongest verifier-free average, reaching 0.388 EM and 0.480 F1. It matches the gold-supervised GRPO outcome baseline in average EM and is slightly higher in average F1, while still trailing on some individual tasks such as HotpotQA EM and MuSiQue EM. Against the gold-supervised IGPO turn-level baseline (0.401 EM / 0.493 F1), SIOP trails by 1.3\% EM and 1.3\% F1 on average, a small remaining supervision gap given that SIOP uses no gold labels at training time. The F1 advantage over GRPO plausibly stems from SIOP's soft semantic-cluster supervision rewarding equivalent surface forms, while the dense per-turn signal compensates for GRPO's sparse outcome reward at scale. These results suggest that the self-induced turn-level signal transfers to the larger backbone, although per-task variation remains.

\vspace{-0.5em}
\subsection{Training Dynamics Analysis}
\vspace{-0.5em}
\label{sec:training_dynamics}

Figure~\ref{fig:behavior_training_curves} tracks search turns, response length, response-length clipping, and generation entropy over 300 optimization steps, testing whether turn-level credit changes behavior beyond final accuracy. In Fig.~\ref{fig:behavior_training_curves}(a), SIOP initially increases tool use and then stabilizes below the verifier-free baselines, suggesting controlled retrieval rather than search suppression. In Fig.~\ref{fig:behavior_training_curves}(b, c), SIOP produces shorter responses with lower clip ratios, indicating less redundant continuation under the response budget. In Fig.~\ref{fig:behavior_training_curves}(d), SIOP reaches a lower-entropy generation regime. Overall, the dynamics show that SIOP reduces redundant search and continuation while preserving end-task quality.

\begin{table}[t]
    \centering
    \small
    \caption{Credit-assignment ablation. Full SIOP and \emph{w/o turn-level advantage} share the same SIOP per-turn process reward, and only the advantage construction differs. GRPO uses gold-answer EM with broadcast advantage as a gold-supervised outcome baseline (not bolded). \textbf{Bold} marks the better entry between the two SIOP-reward rows.}
    \resizebox{\linewidth}{!}{
    \begin{tabular}{p{4cm}|ccccccc|c}
    \toprule
    \multirow{2.4}{*}{Method} & \multicolumn{3}{c}{Single-Hop QA} & \multicolumn{4}{c|}{Multi-Hop QA} & \multirow{2.4}{*}{\textbf{Avg.}} \\
    \cmidrule(lr){2-4}\cmidrule(lr){5-8}
     & NQ & TriviaQA & PopQA & HotpotQA & 2Wiki & MuSiQue & Bamboogle & \\
    \midrule
    \emph{Accuracy (Exact Match)} &&&&&&&& \\
    
    Full SIOP                        & \textbf{0.281} & \textbf{0.604} & 0.416          & 0.344          & \textbf{0.367} & \textbf{0.100} & 0.408          & \textbf{0.360} \\
    \;w/o turn-level advantage       & 0.252          & 0.602          & \textbf{0.438} & \textbf{0.348} & 0.287          & 0.098          & \textbf{0.432} & 0.351 \\
    \midrule
    \emph{Token-level F1} &&&&&&&& \\
    
    Full SIOP                        & \textbf{0.400} & \textbf{0.691} & 0.465          & \textbf{0.443} & \textbf{0.411} & \textbf{0.179} & \textbf{0.528} & \textbf{0.445} \\
    \;w/o turn-level advantage       & 0.391          & 0.680          & \textbf{0.489} & 0.428          & 0.333          & 0.164          & 0.510          & 0.428 \\
    \bottomrule
    \end{tabular}
    }
    \label{tab:abl_credit}
    \vspace{-1.5em}
    \end{table}

\subsection{Ablation Studies}
\vspace{-0.5em}
\noindent\textbf{Component ablation.} Table~\ref{tab:abl_components} removes multi-reference cluster support, reliability calibration, both together, or the per-turn process reward ($\lambda{=}0$) from Full SIOP. The $\lambda{=}0$ row removes the per-turn process reward, so intermediate turns receive zero immediate reward. The calibrated $q_\theta(c\mid q)$ terminal reward still flows backward through the turn-level advantage construction.
Dropping calibration, multi-reference, or both lowers average EM (-1.3\%, -1.7\%, -2.0\%) and F1 (-1.4\%, -1.8\%, -2.1\%), showing that both structural ingredients matter. Multi-reference support carries most multi-hop gains (HotpotQA and 2Wiki on both metrics, MuSiQue F1), while calibration mainly sharpens single-hop factoid answers (\emph{w/o multi-ref} wins NQ and TriviaQA F1). Removing the process reward costs 1.4\% EM and 1.6\% F1 on average, mainly on multi-hop bridging tasks (2Wiki -6.0\% EM, Bamboogle -4.0\% EM), so process rewards help most when credit spans several informative turns.

\noindent\textbf{Credit-assignment mechanism.} Table~\ref{tab:abl_credit} keeps the SIOP per-turn process reward fixed and replaces only the turn-level advantage with a GRPO-style broadcast: per-turn rewards are summed into one trajectory scalar and copied to every token (\emph{w/o turn-level advantage}).
With the reward held fixed, broadcasting the advantage drops EM from 36.0\% to 35.1\% and F1 from 44.5\% to 42.8\%, close to the outcome-only EMPO baseline in Table~\ref{tab:prelim_main} (35.3\% EM / 42.7\% F1). The drop concentrates on multi-hop benchmarks (2Wiki, MuSiQue), indicating that the gains come from turn-level credit assignment rather than denser reward shaping alone.

\section{Conclusion}
\label{sec:conclusion}
\vspace{-0.5em}

In this work, we have introduced verifier-free turn-level credit assignment for agentic reinforcement learning and instantiated it with Self-Induced Outcome Potential (SIOP). SIOP treats clustered final answers as latent future states and rewards each turn by the prefix-conditioned support gain it produces for the rollout's realized mode, weighted by a reliability-calibrated belief over modes. On seven question-answering benchmarks with Qwen3-4B and Qwen3-8B backbones, SIOP outperforms outcome-only verifier-free baselines and approaches the average performance of a gold-supervised outcome baseline without using gold labels during training.


\section{Acknowledgement}

The research work described in this paper was conducted in the JC STEM Lab of Smart City funded by The Hong Kong Jockey Club Charities Trust under Contract 2023-0108.  The work was supported in part by the Hong Kong SAR Government under the Global STEM Professorship and Research Talent Hub. The work of S. Hu was  supported in part by the Hong Kong Innovation and Technology Commission under InnoHK Project CIMDA.

{
    \small
    \bibliographystyle{unsrtnat}
    \bibliography{ref, ref2, ref3}
}

\appendix

\newpage
\begin{center}
    \Large \textbf{Appendix}
\end{center}

\section*{Content}
{\hypersetup{linkcolor=c1}\color{c1}
\hyperref[sec:related]{A. Related Work} \dotfill \pageref{sec:related}\\
\makebox[1em][l]{} \hyperref[sec:related_step_credit]{A.1 Outcome RL and Supervised Step-Level Credit Assignment} \dotfill \pageref{sec:related_step_credit}\\
\makebox[1em][l]{} \hyperref[sec:related_label_free]{A.2 Label-Free Reinforcement Learning from Output Distributions} \dotfill \pageref{sec:related_label_free}\\
\makebox[1em][l]{} \hyperref[sec:related_search_agents]{A.3 Open-Ended Search and Tool-Using Agents} \dotfill \pageref{sec:related_search_agents}\\
\hyperref[app:setup]{B. Implementation Details} \dotfill \pageref{app:setup}\\
\makebox[1em][l]{} \hyperref[app:notation]{B.1 Method Notation} \dotfill \pageref{app:notation}\\
\makebox[1em][l]{} \hyperref[app:training_details]{B.2 Training Details} \dotfill \pageref{app:training_details}\\
\makebox[1em][l]{} \hyperref[app:retrieval_details]{B.3 Retrieval Details} \dotfill \pageref{app:retrieval_details}\\
\makebox[1em][l]{} \hyperref[app:prompt_details]{B.4 Prompt Details} \dotfill \pageref{app:prompt_details}\\
\makebox[1em][l]{} \hyperref[app:semantic_clustering]{B.5 Semantic Clustering} \dotfill \pageref{app:semantic_clustering}\\
\makebox[1em][l]{} \hyperref[alg:cluster]{B.6 Greedy Bidirectional-Entailment Clustering} \dotfill \pageref{alg:cluster}\\
\makebox[1em][l]{} \hyperref[tab:training_hyperparameters]{B.7 Training Hyperparameters} \dotfill \pageref{tab:training_hyperparameters}\\
\makebox[1em][l]{} \hyperref[fig:system_prompt]{B.8 Task Prompt Template} \dotfill \pageref{fig:system_prompt}\\
\hyperref[app:algorithm]{C. Training Algorithm} \dotfill \pageref{app:algorithm}\\
\makebox[1em][l]{} \hyperref[alg:siop]{C.1 Training with SIOP} \dotfill \pageref{alg:siop}\\
\hyperref[app:proofs]{D. Theoretical Characterizations and Proofs} \dotfill \pageref{app:proofs}\\
\makebox[1em][l]{} \hyperref[app:proof_telescoping]{D.1 Proof of Proposition~\ref{prop:telescoping}} \dotfill \pageref{app:proof_telescoping}\\
\makebox[1em][l]{} \hyperref[app:proof_reduction]{D.2 Proof of Proposition~\ref{prop:reduction}} \dotfill \pageref{app:proof_reduction}\\
\makebox[1em][l]{} \hyperref[app:bayesian_calibration]{D.3 Bayesian View and Reliability Calibration} \dotfill \pageref{app:bayesian_calibration}\\
\hyperref[app:additional_results]{E. Additional Results} \dotfill \pageref{app:additional_results}\\
\makebox[1em][l]{} \hyperref[app:overhead]{E.1 Compute Overhead} \dotfill \pageref{app:overhead}\\
\makebox[1em][l]{} \hyperref[app:case_study]{E.2 Case Study} \dotfill \pageref{app:case_study}\\
\hyperref[app:reward_hacking]{F. Reward Hacking and Confident-but-Wrong Modes} \dotfill \pageref{app:reward_hacking}\\
\hyperref[app:limitations]{G. Limitations and Future Work} \dotfill \pageref{app:limitations}\\
}
\vspace{1em}

\section{Related Work}
\label{sec:related}

\subsection{Outcome RL and Supervised Step-Level Credit Assignment}
\label{sec:related_step_credit}

Modern reasoning-model training typically starts from outcome-level RL with verifiable rewards. DeepSeekMath \cite{shao_deepseekmath_2024} introduced GRPO as a critic-free optimizer that compares sampled responses within each prompt group, and DeepSeek-R1 \cite{guoDeepSeekR1IncentivizesReasoning2025} showed that outcome rewards can elicit long chain-of-thought behavior at scale. DAPO \cite{yu_dapo_2025} and GSPO \cite{zheng_group_2025} further improve this group-based paradigm by addressing sampling, clipping, loss normalization, and importance-weighting stability. This line clarifies how to optimize a policy when an outcome reward is already available. It does not by itself specify what the reward target should be without a task-specific verifier, or how that target should be assigned across interaction turns.

The complementary literature studies whether intermediate reasoning steps should receive their own supervision. Lightman et al.\ \cite{lightman_lets_2023} showed that process supervision can outperform outcome supervision for mathematical reasoning. Zhang et al.\ \cite{zhang_entropy-regularized_2024} generalized process reward models through entropy-regularized MDPs, and rStar-Math \cite{guan_rstar-math_nodate} uses MCTS-based self-evolved process preference models. Other methods refine the temporal granularity of RL itself: VinePPO \cite{kazemnejad_vineppo_nodate} uses Monte Carlo rollouts to replace learned value functions for unbiased credit assignment, SPO \cite{guo_segment_2025} operates at a segment level between tokens and trajectories, TEMPO \cite{tran_exploiting_2025} derives nonparametric branch-gated corrections from prefix trees, and Zeng et al.\ \cite{zeng_reinforcing_2025} estimate turn-level advantages for multi-turn tool-use agents modeled as MDPs. These approaches show that finer credit is useful, but they still assume access to supervised step labels, verifiable outcomes, or an externally supplied task reward.

For multi-turn agents, information-potential methods make the dependence on a target especially explicit. IGPO \cite{wang2025igpo} rewards a turn by the marginal increase in the policy's probability of the gold answer. $\Delta$Belief-RL \cite{auzina2026intrinsiccreditassignmentlong} uses belief updates for long-horizon credit assignment. TIPS \cite{xie2026tips} defines an information-potential reward for search-augmented LLMs, InfoPO \cite{kong2026infopo} adapts the idea to user-centric settings, and InfoReasoner \cite{hu2026optimizing} optimizes retrieval-augmented reasoning with a synthetic semantic information-gain reward over retrieval trajectories. These methods establish that turn-level shaping is important for long-horizon optimization, but their potentials are still anchored to a gold answer, supervised belief state, externally specified utility, or retrieval-specific information objective. SIOP changes the role of the target: it constructs a query-specific target space from sampled final answers, treats semantic outcome modes as latent future states, and computes turn rewards as potential changes over this self-induced state space. The supervised gold-answer case is therefore a limiting case rather than the starting assumption (Proposition~\ref{prop:reduction}).

\subsection{Label-Free Reinforcement Learning from Output Distributions}
\label{sec:related_label_free}

Label-free reasoning optimization tries to replace benchmark labels or task-specific verifiers with signals induced by the model's own samples. One family relies on agreement structure across multiple outputs. Semantic entropy by Kuhn et al.\ \cite{kuhn2023semantic} and Farquhar et al.\ \cite{farquhar_detecting_2024} estimates uncertainty over meaning-equivalent answer clusters, while TTRL \cite{zuo_ttrl_2025} uses majority behavior from repeated sampling as a self-generated label. EMPO \cite{zhang2025right} directly optimizes semantic clusters by minimizing semantic entropy, SEED-GRPO \cite{chen_seed-grpo_2025} uses semantic entropy to modulate policy updates on high-uncertainty prompts, and Agarwal et al.\ \cite{agarwal2025unreasonable} show that entropy minimization alone can be surprisingly effective. Xu et al.\ \cite{xu2025exploration} further study entropy scheduling from exploration to exploitation under noisy supervision. These methods establish that output distributions contain usable training signal, but they also expose a failure mode: high agreement is not always correctness. RESTRAIN \cite{yu_restrain_2025} addresses this issue by penalizing spurious consensus rather than trusting majority votes unconditionally.

A second family uses internal model feedback, reference likelihood, or self-evaluation as the reward source. RLPR \cite{yu_rlpr_2025} uses token probabilities assigned to reference answers, NOVER \cite{liu_nover_2025} studies verifier-free incentive training from standard supervised data, DRO \cite{xu_direct_2025} uses reasoning-reflection signals, VeriFree \cite{zhou2025reinforcing} optimizes reference-answer likelihood without an external verifier, and Intuitor \cite{zhao2025learning} replaces external rewards with self-certainty. Related self-improvement methods obtain training signal through self-rewarding or self-play: Self-Rewarding Language Models \cite{yuan2024self} use LLM-as-a-Judge rewards, SPIN \cite{chen2024self} frames alignment as self-play, SPIRAL \cite{liu2018spiral} and SPICE \cite{liu2025spice} extend self-play to multi-turn settings, and Absolute Zero \cite{chen2025absolute} and R-Zero \cite{huang2025rzero} study self-evolution from zero external data.

The common limitation for our setting is not the absence of useful unlabeled signals, but their granularity. TTRL \cite{zuo_ttrl_2025}, EMPO \cite{zhang2025right}, SEED-GRPO \cite{chen_seed-grpo_2025}, RLPR \cite{yu_rlpr_2025}, NOVER \cite{liu_nover_2025}, VeriFree \cite{zhou2025reinforcing}, and Intuitor \cite{zhao2025learning} primarily attach rewards to prompts, final answers, or complete trajectories. They help decide which sampled outcomes or prompts are more reliable, but they do not identify which intermediate interaction turns moved the trajectory toward those outcomes. SIOP uses the same kind of unlabeled distributional structure for a different object: semantic outcome modes define a query-specific latent future-state space, and process rewards are computed as potential changes over that space. This turns outcome-level reliability into turn-level credit assignment rather than only outcome selection.

\subsection{Open-Ended Search and Tool-Using Agents}
\label{sec:related_search_agents}

Search and tool-use agents make credit assignment harder than single-turn generation because the model must decide not only what to answer, but also when to retrieve, which evidence to trust, and when to stop interacting with the environment. The agentic RAG survey by Li et al.\ \cite{li_towards_2025} summarizes this shift from static retrieval to iterative information seeking. ReSearch \cite{chen2025researchlearningreasonsearch} and Scent of Knowledge \cite{qian2025scentknowledgeoptimizingsearchenhanced} further show that search behavior itself can be optimized rather than treated as a fixed preprocessing step. Open-domain QA benchmarks such as PopQA \cite{mallen-etal-2023-trust}, HotpotQA \cite{yang-etal-2018-hotpotqa}, and MuSiQue \cite{trivedi-etal-2022-musique} expose the same issue empirically: final-answer quality depends on intermediate evidence acquisition, especially for long-tail and multi-hop questions.

Reward design has therefore become central in tool-use and search-augmented RL. ToolRL \cite{qian_toolrl_2025} studies how reward granularity, scale, and temporal dynamics affect tool-use learning. Process vs. Outcome Reward \cite{zhang2025processvsoutcomereward} shows that finer-grained feedback can improve long-horizon agentic RAG, while also emphasizing the cost of obtaining gold intermediate supervision. These results support the need for turn-level feedback, but they still leave open how to obtain such feedback when neither intermediate labels nor task-specific answer verifiers are available.

Recent agentic RL methods adapt group-based optimization to multi-step environments. Search-R1 \cite{jin2025searchr} and ReSearch \cite{chen_research_2025} train models to interleave reasoning with search, while GiGPO \cite{feng_group--group_2025}, Step-GRPO \cite{peng_learning_2025}, and Tree-GRPO \cite{ji_tree_2025} construct group comparisons, step-level advantages, or tree-structured credit over agent trajectories. These methods improve optimization for search agents, but their learning signal still typically comes from a task reward, answer verifier, or environment-defined success criterion.
SIOP addresses a complementary supervision gap. Instead of assuming that the environment can verify the final answer or score each tool call, it constructs a latent outcome space from unlabeled sampled answers and uses potential changes over that space to score intermediate turns. This lets the same output distribution that supports label-free outcome selection also guide which retrieval decisions and tool calls were informative.

\section{Implementation Details}
\label{app:setup}

\phantomsection\label{app:notation}
\noindent\textbf{Method notation.} Table~\ref{tab:notation_core} lists the recurring symbols used in the method section. Local quantities used only in one equation are defined in place.

\begin{table}[t]
    \centering
    \footnotesize
    \caption{Core notation used in the method section.}
    \label{tab:notation_core}
    \begingroup
    \renewcommand{\arraystretch}{1.08}
    \begin{tabular}{@{}p{0.22\linewidth}p{0.72\linewidth}@{}}
    \toprule
    Symbol & Meaning \\
    \midrule
    $q$ & Input query. \\
    $K$ & Number of sampled rollouts for the same query. \\
    $\tau^{(k)}, \tau_{\le t}^{(k)}$ & Rollout $k$ and its prefix after turn $t$. \\
    $T_k$ & Number of assistant turns in rollout $k$. \\
    $y^{(k)}$ & Final answer emitted by rollout $k$. \\
    $c, c(k)$ & Semantic outcome mode and the mode assigned to rollout $k$. \\
    $\mathcal C(q)$ & Semantic outcome modes induced by the sampled rollouts for query $q$. \\
    $S_c$ & Set of rollouts assigned to semantic mode $c$. \\
    $R(c), \alpha_y^{(c)}$ & Reference answers for mode $c$ and their normalized weights. \\
    $m(c\mid q)$ & Empirical mass of mode $c$, computed from rollout counts. \\
    $\mathrm{obs}_t^{(k)}$ & Tool observation returned at turn $t$ of rollout $k$. \\
    $s_{\mathrm{NLI}}(\mathrm{obs},y)$ & Entailment confidence that observation $\mathrm{obs}$ supports answer $y$. \\
    $r(c,q)$ & Reliability score for mode $c$, instantiated as evidence support in search-agent experiments. \\
    $\eta$ & Strength of the reliability calibration applied to empirical mode mass. \\
    $q_\theta(c\mid q)$ & Reliability-calibrated target distribution over semantic outcome modes. \\
    $\tilde p_\theta(c\mid\tau_{\le t})$ & Multi-reference approximation to prefix-conditioned support for mode $c$. \\
    $\Phi_\theta(\tau_{\le t}^{(k)},c(k))$ & Outcome potential for prefix $\tau_{\le t}^{(k)}$ and rollout $k$'s mode. \\
    $r_{t,\mathrm{proc}}^{(k)}$ & Process reward from the change in outcome potential at turn $t$. \\
    $\lambda$ & Mixing coefficient between process reward and terminal outcome mass. \\
    $\bar r_t^{(k)}$ & Immediate turn reward after adding terminal outcome mass at the final answer turn. \\
    $\gamma$ & Discount factor used to propagate later rewards backward. \\
    $\tilde r_t^{(k)}, A_t^{(k)}$ & Group-normalized immediate reward and the discounted turn advantage used for policy optimization. \\
    $\mathcal T_t^{(k)}$ & Trainable policy tokens generated within turn $t$ of rollout $k$. \\
    \bottomrule
    \end{tabular}
    \endgroup
    \vspace{-0.8em}
\end{table}

All training hyperparameters are summarized in Table~\ref{tab:training_hyperparameters}, and the system prompt template is shown in Figure~\ref{fig:system_prompt}. The remainder of this subsection covers the three components of the pipeline that are not a single numeric value: training, retrieval, and semantic clustering.

\phantomsection\label{app:training_details}
\noindent\textbf{Training.} The policy backbones are Qwen3-4B and Qwen3-8B. The actor is trained with the verl framework (v0.8.0.dev) using FSDP and a hybrid-engine weight sync to a vLLM 0.11 rollout backend on $8\times$ NVIDIA A800-80GB, with PyTorch 2.8.0+cu128 and flash-attention 2.8.1. Optimization uses AdamW with a constant learning rate of $10^{-6}$, KL regularization to a frozen reference model (\texttt{low\_var\_kl} estimator, coefficient 0.005), discount factor $\gamma{=}1$, train batch size 128, and a PPO mini-batch size of 32 with a micro-batch size of 4 per GPU. The SIOP hyperparameters are $N{=}3$ references per cluster, $\eta{=}1$ reliability calibration strength, and $\lambda{=}0.5$ process/terminal mixing. Each run is trained for 300 optimizer steps with validation every 50 steps. The compute profile in Table~\ref{tab:compute_overhead} is measured on Qwen3-4B, where a single SIOP run takes 11.75 wall-clock hours on $8\times$A800-80GB, or about 94 GPU-hours.

\phantomsection\label{app:retrieval_details}
\noindent\textbf{Retrieval.} The retriever is a local Wikipedia index built from the \texttt{wiki-18} dump with \texttt{intfloat/e5-base-v2} sentence embeddings and a FAISS flat index over the full corpus. It runs as a standalone HTTP service with FAISS GPU acceleration, independent of the training processes, and each tool call returns the top-3 passages. Tool calls use the Hermes \texttt{<tool\_call>}/\texttt{<tool\_response>} format, with $\texttt{max\_response\_length}{=}2048$ and the tool response length capped at $2048$ tokens per call.

\phantomsection\label{app:prompt_details}
\noindent\textbf{Prompt.} All methods share the same prompt template, as shown in Figure~\ref{fig:system_prompt}. A minimal system message sets the assistant role, and a user message specifies the interaction protocol: reasoning is wrapped in \texttt{<think>}/\texttt{</think>} tags after every new observation, search queries are issued with \texttt{<tool\_call>}/\texttt{</tool\_call>}, retrieval results return between \texttt{<tool\_response>}/\texttt{</tool\_response>}, and the final answer is emitted in \texttt{<answer>}/\texttt{</answer>}. The template allows up to five assistant turns and is identical across SIOP, TTRL, EMPO, and GRPO so that differences in end-task performance are not attributable to prompt engineering.

\phantomsection\label{app:semantic_clustering}
\noindent\textbf{Semantic clustering.} The framework is agnostic to the clustering backend: in structured settings, answers can be normalized by rules or regular expressions before clustering, and in open-ended settings a small judge model or a pairwise entailment model suffices, consistent with semantic-entropy style approaches by Kuhn et al.\ \cite{kuhn2023semantic} and Farquhar et al.\ \cite{farquhar_detecting_2024}. In this paper we instantiate clustering with \texttt{microsoft/deberta-v2-xlarge-mnli} as a bidirectional-entailment judge. For each query's $K$ rollouts we contextualize every answer $y^{(k)}$ with the query as $\text{``}q\;y^{(k)}\text{''}$ before NLI scoring, so equivalence judgments depend on the question rather than surface form alone. Two contextualized answers $a$ and $b$ are merged when $a{\Rightarrow}b$ and $b{\Rightarrow}a$ under a strict bidirectional test: the NLI head must output the entailment class in \emph{both} directions. We then run a one-pass greedy clustering over the $K$ answers (Algorithm~\ref{alg:cluster}): each unassigned answer opens a new cluster, is batch-compared against all still-unassigned later answers, and the ones that bidirectionally entail it are absorbed. The first answer entering a cluster serves as its canonical reference for multi-reference scoring. The same NLI model also produces the reliability signal $r(c,q)$ from retrieved context. It is loaded as a CPU-side singleton so it never contends with the vLLM rollout backend for GPU memory, and clustering is invoked through the same scoring service that computes $\tilde p_\theta$.

\begin{algorithm}[t]
\caption{Greedy bidirectional-entailment clustering}
\label{alg:cluster}
\begin{algorithmic}[1]
\Require answers $\{y^{(k)}\}_{k=1}^K$, query $q$, NLI judge $f_{\mathrm{NLI}}$
\State $\hat y^{(k)} \leftarrow \text{``}q\;y^{(k)}\text{''}$ for $k=1,\dots,K$ \Comment{contextualize with query}
\State $\mathrm{cid}[k] \leftarrow \bot$ for all $k$
\State $c \leftarrow 0$
\For{$k = 1, \dots, K$}
    \If{$\mathrm{cid}[k] \neq \bot$} \State \textbf{continue} \EndIf
    \State $\mathrm{cid}[k] \leftarrow c$
    \State Mark $y^{(k)}$ as the canonical reference for cluster $c$
    \State $\mathcal{U} \leftarrow \{j : j > k,\; \mathrm{cid}[j] = \bot\}$
    \State $e \leftarrow f_{\mathrm{NLI}}(\hat y^{(k)}, \hat y^{(j)})$ and $e' \leftarrow f_{\mathrm{NLI}}(\hat y^{(j)}, \hat y^{(k)})$ for all $j \in \mathcal{U}$ \Comment{batched}
    \For{each $j \in \mathcal{U}$ with $e_j = \text{entail}$ and $e'_j = \text{entail}$}
        \State $\mathrm{cid}[j] \leftarrow c$
    \EndFor
    \State $c \leftarrow c + 1$
\EndFor
\State \Return $\mathrm{cid}$, cluster canonical references
\end{algorithmic}
\end{algorithm}

\begin{table}[t]
    \centering
    \caption{Key hyperparameters used for SIOP training. Values apply to all ablations and baselines unless otherwise noted.}
    \begin{tabular}{l|l}
        \toprule
        \textbf{Hyperparameter} & \textbf{Value} \\
        \midrule
        Policy backbone & Qwen3-4B/8B \\
        Training Batch Size & 128 \\
        Validation Batch Size & 128 \\
        Max Prompt Length & 1024 \\
        Max Response Length & 2048 \\
        PPO Mini-Batch Size & 32 \\
        PPO Micro-Batch Size per GPU & 4 \\
        Actor Learning Rate & $1 \times 10^{-6}$ \\
        KL Loss Coefficient & 0.005 \\
        KL Loss Estimator & low\_var\_kl \\
        Discount Factor $\gamma$ & 1.0 \\
        Training Steps & 300 \\
        Save / Validation Frequency & every 50 steps \\
        Rollouts per Query $K$ & 4 \\
        Rollout Temperature & 1.0 \\
        Max Assistant Turns (Training) & 5 \\
        Max Tool Response Length & 2048 \\
        Retrieval Top-$k$ & 3 \\
        Tensor Parallel Size (rollout) & 4 \\
        Rollout GPU Memory Utilization & 0.4 \\
        Scorer GPU Memory Utilization & 0.15 \\
        Scorer Max Model Length & 8192 \\
        SIOP Process-Terminal Mix $\lambda$ & 0.5 \\
        SIOP References per Cluster $N$ & 3 \\
        SIOP Reliability Strength $\eta$ & 1.0 \\
        NLI Strict Entailment & true \\
        NLI Device & CPU \\
        GPUs per Node & 8 \\
        \bottomrule
    \end{tabular}
    \label{tab:training_hyperparameters}
\end{table}

\begin{figure}[t]
    \centering
    \begin{tcolorbox}[
        colback=c1!5,
        colframe=c1!60,
        title=\textbf{Task Prompt Template},
        fonttitle=\color{white},
        boxrule=1pt,
        arc=4pt,
        left=6pt,
        right=6pt,
        top=6pt,
        bottom=6pt,
        enhanced,
        shadow={2pt}{-2pt}{0pt}{gray!30}
    ]
    \footnotesize
    \textbf{System:} You are a helpful and harmless assistant.

    \medskip
    \textbf{User:} Answer the given question. You must conduct reasoning inside \texttt{<think>} and \texttt{</think>} first every time you get new information. After reasoning, if you find you don't know the exact answer, you can just call a search engine by \texttt{<tool\_call>} query \texttt{</tool\_call>} (for example, suppose you don't know where the capital of china is, you can just output \texttt{<tool\_call>capital china</tool\_call>} to get a search) and it will return the top searched results between \texttt{<tool\_response>} and \texttt{</tool\_response>}. You can search as many times as you want. If you find no further external knowledge needed, you can directly provide the answer inside \texttt{<answer>} and \texttt{</answer>}, without detailed illustrations. For example, \texttt{<answer>} Beijing \texttt{</answer>}. Question: \{question\}.
    \end{tcolorbox}
    \caption{Task prompt template used for training and inference. All rollouts see an identical system message and a user message that combines the retrieval-augmented reasoning instructions with the query.}
    \label{fig:system_prompt}
    \vspace{-1em}
\end{figure}

\section{Training Algorithm}
\label{app:algorithm}

Algorithm~\ref{alg:siop} summarizes the full training pipeline. Each iteration processes one batch of queries in four phases. The \emph{rollout and clustering} phase samples $K$ full rollouts per query, clusters their final answers into semantic outcome modes via the NLI judge, and computes the reliability-calibrated target distribution $q_\theta(c\mid q)$ over those modes. The \emph{scoring} phase runs a single forward pass on a frozen copy of the policy to obtain $\log p_\theta(y\mid\tau_{\le t})$ for every $(y, \tau_{\le t})$ pair, aggregates them into the multi-reference cluster support $\tilde p_\theta(c\mid\tau_{\le t})$, and turns consecutive differences into per-turn process rewards. The \emph{advantage construction} phase adds terminal outcome mass at the final answer turn, pools the augmented rewards across the query group, computes the group-relative $z$-score, and propagates it backward through the discounted turn-level advantage. The \emph{policy update} phase broadcasts the turn-level advantage onto the trainable tokens within each turn and applies the GRPO-style clipped objective. The scoring pass is served by a GPU-side inference service sharing the rollout backend's memory budget, while the NLI calls run through a CPU-side singleton. Thus the main additional GPU cost over outcome-only RL is the forward pass over pseudo-response pairs at each training step.

\begin{algorithm}[t]
\caption{Training with SIOP}
\label{alg:siop}
\begin{algorithmic}[1]
\Require Query batch $\mathcal{Q}$, policy $\pi_\theta$, frozen scoring copy, rollout count $K$, discount $\gamma$, mixture weight $\lambda$
    \For{each query $q \in \mathcal{Q}$}
        \State \textcolor{gray}{// Phase 1: Rollout and semantic outcome mode construction}
        \State Sample $K$ complete rollouts $\{\tau^{(k)}\}_{k=1}^K$
        \State Extract final answers $\{y^{(k)}\}_{k=1}^K$ and cluster them into semantic outcome modes $\mathcal{C}(q)$
        \State \textcolor{gray}{// Phase 2: Build a reliability-calibrated target over latent future states}
        \For{each semantic mode $c \in \mathcal{C}(q)$}
            \State Set $S_c=\{k:c(k)=c\}$ and empirical mass $m(c\mid q)=|S_c|/K$
            \State Select representative references $R(c)\subseteq c$ and normalized weights $\{\alpha_y^{(c)}\}_{y\in R(c)}$
            \State Estimate mode reliability $r(c,q)$ from semantic or evidence support of $R(c)$
        \EndFor
        \State Form the reliability-calibrated target over latent future states
        \Statex \hspace{\algorithmicindent} $q_\theta(c\mid q) \propto m(c\mid q)\exp(\eta r(c,q))$ for all $c\in\mathcal{C}(q)$
        \For{each rollout $\tau^{(k)}$}
            \State \textcolor{gray}{// Phase 3: Score each turn by the change in support for its realized mode}
            \For{turn $t=1,\dots,T_k$}
                \State Score each $y\in R(c(k))$ under prefixes $\tau_{\le t}^{(k)}$ and $\tau_{\le t-1}^{(k)}$ using the frozen copy
                \State \textcolor{gray}{// Multi-reference support approximates the full semantic cluster mass}
                \State Aggregate scores into multi-reference cluster supports
                \Statex \hspace{\algorithmicindent} $\tilde p_\theta(c(k)\mid \tau_{\le t}^{(k)})=\sum_{y\in R(c(k))}\alpha_y^{(c(k))}p_\theta(y\mid \tau_{\le t}^{(k)})$
                \Statex \hspace{\algorithmicindent} $\tilde p_\theta(c(k)\mid \tau_{\le t-1}^{(k)})=\sum_{y\in R(c(k))}\alpha_y^{(c(k))}p_\theta(y\mid \tau_{\le t-1}^{(k)})$
                \State Compute the outcome potentials $\Phi_\theta(\tau_{\le t}^{(k)},c(k))$ and $\Phi_\theta(\tau_{\le t-1}^{(k)},c(k))$
                \State Compute $r_{t,\mathrm{proc}}^{(k)}=\Phi_\theta(\tau_{\le t}^{(k)},c(k))-\Phi_\theta(\tau_{\le t-1}^{(k)},c(k))$
                \State Set $\bar r_t^{(k)} = \lambda r_{t,\mathrm{proc}}^{(k)}$
            \EndFor
            \State Add terminal outcome mass at the answer turn:
            $\bar r_{T_k}^{(k)} \leftarrow \bar r_{T_k}^{(k)} + (1-\lambda) r_{\mathrm{final}}^{(k)}$
        \EndFor
    \State \textcolor{gray}{// Phase 4: Convert augmented rewards into turn-level advantages}
    \State Pool all $\{\bar r_t^{(k)}\}_{k,t}$ within the query group and $z$-normalize using the group's own mean $\mu_g$ and std $\sigma_g$ to obtain $\{\tilde r_t^{(k)}\}_{k,t}$
    \For{each rollout $\tau^{(k)}$}
        \State Compute turn advantages $A_t^{(k)} = \sum_{u=t}^{T_k}\gamma^{u-t}\tilde r_u^{(k)}$
        \State Assign $A_t^{(k)}$ only to trainable policy tokens generated in turn $t$
    \EndFor
\EndFor
\State \textcolor{gray}{// Phase 5: Optimize only policy-generated tokens with turn-specific credit}
\State Update $\pi_\theta$ with the SIOP objective
\end{algorithmic}
\end{algorithm}

\section{Theoretical Characterizations and Proofs}
\label{app:proofs}

\subsection{Proof of Proposition~\ref{prop:telescoping}}
\label{app:proof_telescoping}

By definition of the outcome potential and the process reward,
\begin{equation}
\begin{aligned}
r_{t,\mathrm{proc}}^{(k)}
&=
\Phi_\theta(\tau_{\le t}^{(k)}, c(k))
-\Phi_\theta(\tau_{\le t-1}^{(k)}, c(k)) \\
&=
q_\theta(c(k) \mid q)\Bigl[
\log\!\bigl(\tilde p_\theta(c(k) \mid \tau_{\le t}^{(k)})\bigr) -
\log\!\bigl(\tilde p_\theta(c(k) \mid \tau_{\le t-1}^{(k)})\bigr)
\Bigr].
\end{aligned}
\end{equation}
Summing over turns yields a telescoping series:
\begin{equation}
\sum_{t=1}^{T_k} r_{t,\mathrm{proc}}^{(k)}
=
\sum_{t=1}^{T_k}
\Bigl[
\Phi_\theta(\tau_{\le t}^{(k)}, c(k))
-\Phi_\theta(\tau_{\le t-1}^{(k)}, c(k))
\Bigr],
\end{equation}
and all interior terms cancel. The remaining boundary terms are
\begin{equation}
\begin{aligned}
\Phi_\theta(\tau_{\le T_k}^{(k)}, c(k))
&-\Phi_\theta(\tau_{\le 0}^{(k)}, c(k)) \\
&=
q_\theta(c(k) \mid q)\Bigl[
\log\!\bigl(\tilde p_\theta(c(k) \mid \tau_{\le T_k}^{(k)})\bigr) -
\log\!\bigl(\tilde p_\theta(c(k) \mid q)\bigr)
\Bigr].
\end{aligned}
\end{equation}
Thus the cumulative reward depends on the improvement in final conditional support, while each turn's local credit is determined only by its own marginal contribution to that improvement.

\subsection{Proof of Proposition~\ref{prop:reduction}}
\label{app:proof_reduction}

\textbf{Factorization.} The factored form follows by expanding the difference of two outcome potentials in Section~\ref{sec:method}. The calibrated target-mass factor $q_\theta(c(k)\mid q)$ is query-specific and constant over turns of the same rollout, so it can be grouped outside the log-difference of the multi-reference cluster support. The remaining factor is the turn-wise change in log cluster support for the realized cluster $c(k)$. Together they give ``target mass times cluster-level support gain,'' which is the stated factorization.

\textbf{Supervised limit.} When a gold answer $y^\star$ is available and the self-induced target distribution concentrates on the gold cluster $c^\star$, the assumptions $q_\theta(c^\star \mid q)=1$, $R(c^\star)=\{y^\star\}$, and $\alpha_{y^\star}^{(c^\star)}=1$ imply $\tilde p_\theta(c^\star \mid \tau_{\le t}^{(k)}) = p_\theta(y^\star \mid \tau_{\le t}^{(k)})$. Substituting into the factored reward collapses the target mass to $1$ and the cluster support to the gold-answer likelihood, yielding
\begin{equation}
r_t^{(k)} = \log p_\theta(y^\star \mid \tau_{\le t}^{(k)}) - \log p_\theta(y^\star \mid \tau_{\le t-1}^{(k)}),
\end{equation}
which is the standard supervised log-support reward used by turn-level shaping methods.

\subsection{Bayesian View and Reliability Calibration}
\label{app:bayesian_calibration}

The self-induced target distribution can be read as a calibrated belief over semantic outcome modes obtained by a Bayesian-style update. The empirical mass $m(c\mid q)$ is the rollout-induced prior over latent future states, while the factor $\exp(u_\theta(c,q))$ occupies the algebraic position of an evidence likelihood derived from generic reliability signals. Under this view,
\begin{equation}
q_\theta(c\mid q)
\propto
\underbrace{m(c\mid q)}_{\text{rollout-induced prior}}
\underbrace{\exp(u_\theta(c,q))}_{\text{reliability likelihood proxy}},
\end{equation}
where the reliability strength is already absorbed into $u_\theta(c,q)=\eta r(c,q)$. This interpretation is not a claim that the NLI score is a calibrated likelihood. It only states that SIOP uses reliability signals in the same algebraic position as an evidence term before normalizing over semantic outcome modes.

\begin{proposition}[Reliability calibration can reverse a spurious majority]
\label{prop:calibration_reversal}
Consider two semantic outcome modes $c_s$ and $c_r$ with empirical masses $m_s>m_r>0$, where $c_s$ is the higher-frequency mode and $c_r$ is the lower-frequency mode. Let $u_\theta(c,q)=\eta r(c,q)$ with $\eta>0$. The calibrated target assigns larger mass to $c_r$ than to $c_s$ if and only if
\begin{equation}
r(c_r,q)-r(c_s,q)
>
\frac{1}{\eta}\log\frac{m_s}{m_r}.
\end{equation}
Thus a lower-frequency mode can overtake a raw majority exactly when its reliability advantage exceeds the log frequency disadvantage scaled by $1/\eta$.
\end{proposition}

\begin{proof}
The normalized denominator of $q_\theta(\cdot\mid q)$ is shared by both modes, so $q_\theta(c_r\mid q)>q_\theta(c_s\mid q)$ is equivalent to
\begin{equation}
m_r\exp(\eta r(c_r,q))
>
m_s\exp(\eta r(c_s,q)).
\end{equation}
Taking logarithms and rearranging gives
\begin{equation}
\eta\bigl(r(c_r,q)-r(c_s,q)\bigr)
>
\log\frac{m_s}{m_r},
\end{equation}
which is equivalent to the stated condition. If the inequality is reversed, the higher-frequency mode remains preferred. Equality gives tied calibrated mass.
\end{proof}

The right-hand side $(1/\eta)\log(m_s/m_r)$ is the frequency disadvantage of $c_r$ measured in the units of the chosen reliability score. Larger $\eta$ makes reliability evidence sharper and lowers the threshold a low-frequency mode must clear. This condition should not be read as a guarantee that every wrong majority is corrected. If $r(c,q)$ is a bounded NLI confidence, then highly imbalanced majorities may be impossible to reverse under a fixed calibration strength. The proposition instead identifies the exact regime in which reliability evidence can overcome raw self-consistency.

\section{Additional Results}
\label{app:additional_results}

\subsection{Compute Overhead}
\label{app:overhead}

We profile SIOP and GRPO under identical settings (Qwen3-4B, 8${\times}$A800-80GB, same retriever, same agent loop, $K{=}4$ rollouts, 300 optimizer steps) and report in Table~\ref{tab:compute_overhead} the two phases where the per-step wall-clock time differs materially. Rollout generation, reference-model forward passes, and the FSDP actor update are similar within measurement noise across the two methods. SIOP's overhead is concentrated in two components. The first is \emph{scoring} (average 22.4\,s per step), dominated by the pseudo-response forward pass that evaluates $\log p_\theta(y\mid \tau_{\le t})$ at every turn boundary. The second is \emph{advantage computation} (average 6.3\,s per step), where turn-level per-query-group $z$-normalization and backward discounted accumulation across turn boundaries replace the single vectorized broadcast used by GRPO. Aggregated across all 300 steps, SIOP takes 11.75 wall-clock hours on 8$\times$A800-80GB (94 GPU-hours per run), while GRPO takes 9.6 wall-clock hours (77 GPU-hours per run). SIOP therefore incurs an additional 2.1 hours and 17 GPU-hours per training run for verifier-free turn-level reward computation.
\begin{figure}[t]
    \centering
    \begin{tcolorbox}[
    colback=c1!5,
    colframe=c1!60,
    title=\textbf{Case Study: SIOP on a multi-hop query},
    fonttitle=\bfseries\color{white},
    boxrule=1pt,
    arc=4pt,
    left=6pt,
    right=6pt,
    top=6pt,
    bottom=6pt,
    enhanced,
    shadow={2pt}{-2pt}{0pt}{gray!30}
]
\footnotesize
\textbf{Query:} When is the director of film Pudhu Vazhvu's birthday? \\
\textbf{Ground-truth:} 1 March 1910 \\
\textbf{SIOP rollout:} \\
\texttt{<think>} To find the birthday of the director of ``Pudhu Vazhvu'', I first need to identify the director. Let me search for the director of this film. \texttt{</think>} \\
\texttt{<tool\_call>} director of Pudhu Vazhvu film \texttt{</tool\_call>} \\
\texttt{<tool\_response>} Doc 1 (Title: ``K. S. Ravikumar''): K. S. Ravikumar (born 30 May 1958) is an Indian film director and actor, primarily working in Tamil cinema\ldots\ Doc 3 (Title: ``Pudhu Vazhvu''): Pudhu Vazhvu is a 1957 Indian Tamil language film directed by M. K. Thyagaraja Bhagavathar. \texttt{</tool\_response>} \\
\texttt{<think>} Doc 3 confirms that ``Pudhu Vazhvu'' was directed by M. K. Thyagaraja Bhagavathar. The retrieved documents do not give his birthday, so I need a second query for that. \texttt{</think>} \\
\texttt{<tool\_call>} birthday M. K. Thyagaraja Bhagavathar \texttt{</tool\_call>} \\
\texttt{<tool\_response>} Doc 1 (Title: ``M. K. Thyagaraja Bhagavathar''): Mayavaram Krishnamurthy Thyagaraja Bhagavathar (1 March 1910 to 1 November 1959), also called M.K.T., was an Indian actor, producer and Carnatic music singer\ldots \texttt{</tool\_response>} \\
\texttt{<think>} The retrieved document states his birth date as 1 March 1910. \texttt{</think>} \\
\texttt{<answer>} 1 March 1910 \texttt{</answer>}
\end{tcolorbox}
    \caption{Case study of a SIOP rollout on a multi-hop question. The agent issues two search queries: the first to identify the director, the second to retrieve his birthday. The final answer matches the gold label.}
    \label{fig:case_study}
\end{figure}

\begin{table}[t]
\centering
\small
\caption{Per-step wall-clock time for the two phases where SIOP and GRPO differ materially (mean over the last five steps). Rollout generation, reference-model forward passes, and actor updates are similar within measurement noise and omitted.}
\label{tab:compute_overhead}
\begin{tabular}{l|ccc}
\toprule
Phase & SIOP & GRPO & $\Delta$ \\
\midrule
Scoring + NLI (per step) & 22.4\,s & 0.00\,s & +22.4\,s \\
Advantage (per step) & 6.3\,s & 0.03\,s & +6.3\,s \\
\midrule
Total (per step) & 141.0\,s & 115.0\,s & +26.0\,s (+23\%) \\
Wall-clock (300 steps, 1 node) & 11.75\,h & 9.6\,h & +2.1\,h \\
GPU-hours ($8\times$A800 per run) & 94 & 77 & +17 \\
\bottomrule
\end{tabular}
\end{table}

\subsection{Case Study}
\label{app:case_study}

\noindent\textbf{Qualitative rollout.} Figure~\ref{fig:case_study} gives a qualitative trace of a full SIOP rollout on the query \emph{``When is the director of film Pudhu Vazhvu's birthday?''}. The trajectory requires two distinct pieces of evidence: the first search identifies the director as M.\ K.\ Thyagaraja Bhagavathar, and the second search retrieves his birth date. This example illustrates the basic credit-assignment problem. The final answer confirms that the trajectory succeeds, but it does not reveal how much each retrieval step contributed. SIOP instead scores each turn by whether it increases support for the rollout's semantic outcome mode, so the bridge-entity search and the birth-date search can receive different credit.

\noindent\textbf{Multi-reference trace.} This case study illustrates the multi-reference regime on the query \emph{``When is Zeinab Jammeh's husband's birthday?''}. Unlike the previous example, the dominant semantic cluster contains two surface forms of the same date, ``25 May 1965'' and ``May 25, 1965'', and covers only part of the rollout group. A single-prototype reward would be sensitive to which wording is chosen as the reference, while SIOP averages support across both paraphrases. Figure~\ref{fig:case_study_multiref} and Table~\ref{tab:case_quant_multiref} show that this matters for credit assignment. The first retrieval identifies the bridge entity, which is useful for the reasoning path, but it does not yet provide the decisive date evidence. Since SIOP scores prefix support for the final date cluster rather than manually defined subgoals, the cluster-level reward may not be positive until the birth-date evidence appears in the second retrieval. The final answer turn aggregates support for both paraphrases. This case shows that semantic outcome modes are not just labels for final answers, but structured future states that can support robust turn-level rewards across equivalent wordings.

\begin{table}[t]
\centering
\small
\caption{Per-turn diagnostics for the case study with multi-reference in Figure~\ref{fig:case_study_multiref}. The first search returns documents about a different person, so the cluster surrogate $\tilde p_\theta$ drops sharply and yields a negative $r_1$. The second search retrieves the correct husband, partially recovering support. The final answer turn assigns high support to both reference paraphrases, and $r_3$ aggregates the multi-reference log gain plus the terminal mass $(1-\lambda)q_\theta=0.35$. Each cell of ``mean $\log p$ to ref'' lists the two references in order, $y_1=$``25 May 1965'' and $y_2=$``May 25, 1965''.}
\label{tab:case_quant_multiref}
\begin{tabular}{l|cccc}
\toprule
& $\tau_0$ (prompt) & $\tau_1$ (after search 1) & $\tau_2$ (after search 2) & $\tau_3$ (after answer) \\
\midrule
Mean $\log p$ to $y_1$, $y_2$ & $-2.05,\,-1.79$ & $-5.78,\,-3.38$ & $-5.44,\,-3.09$ & $-0.47,\,-0.00$ \\
Cluster support $\tilde p_\theta$ & $0.147$ & $0.019$ & $0.025$ & $0.810$ \\
Process reward $r_t$ & $-$ & $-1.45$ & $+0.20$ & $+2.79$ \\
\bottomrule
\end{tabular}
\vspace{-0.5em}
\end{table}

\begin{figure}[t]
    \centering
    \begin{tcolorbox}[
    colback=c1!5,
    colframe=c1!60,
    title=\textbf{\color{white}Multi-reference case study: SIOP per-turn reward with $N=2$ paraphrases},
    fonttitle=\bfseries\color{white},
    boxrule=1pt,
    arc=4pt,
    left=6pt,
    right=6pt,
    top=6pt,
    bottom=6pt,
    enhanced,
    shadow={2pt}{-2pt}{0pt}{gray!30}
]
\footnotesize
\textbf{Query:} When is Zeinab Jammeh's husband's birthday? \\
\textbf{Ground-truth:} 25 May 1965 (Yahya Jammeh). \\
\textbf{Cluster reference $R(c)$:} \{``25 May 1965'', ``May 25, 1965''\}, $\alpha_1=\alpha_2=0.5$, $N=2$, $q_\theta(c\mid q)=0.70$. \\
\textbf{SIOP rollout:} \\
\texttt{<think>} Zeinab Jammeh might be the wife of a politician in Africa. I will search for her husband's birthday. \texttt{</think>} \\
\texttt{<tool\_call>} zeinab jammeh husband birthday \texttt{</tool\_call>} \\
\texttt{<tool\_response>} Doc 1 (\textit{Ahmed Zewail}): off-topic biography of a chemist\ldots\ Doc 2 (\textit{Yahya Jammeh}): Jammeh briefly dated Tuti Faal\ldots\ Jammeh married his second wife Zeinab\ldots\ Doc 3 (\textit{Yahya Jammeh}): Suma Jammeh, on 26 March 1999\ldots\ \texttt{</tool\_response>} \\
\textcolor{red!70!black}{$\Rightarrow$ $r_1 = -1.45$ (the off-topic Ahmed Zewail document temporarily dominates the prefix and the cluster surrogate falls from $0.147$ to $0.019$).} \\
\texttt{<think>} Doc 2/3 confirm that Yahya Jammeh is Zeinab's husband. Now I need his date of birth. \texttt{</think>} \\
\texttt{<tool\_call>} Yahya Jammeh date of birth \texttt{</tool\_call>} \\
\texttt{<tool\_response>} Doc 1 (\textit{Yahya Jammeh}): Yahya Abdul-Aziz Jemus Junkung Jammeh\ldots\ born 25 May 1965\ldots\ served as the President of the Gambia\ldots\ \texttt{</tool\_response>} \\
\textcolor{green!50!black}{$\Rightarrow$ $r_2 = +0.20$ (the gold date appears in retrieval, and cluster surrogate partially recovers from $0.019$ to $0.025$, with the multi-reference average held back by the second paraphrase ``May 25, 1965'').} \\
\texttt{<think>} The document gives the date directly. \texttt{</think>} \\
\texttt{<answer>} May 25, 1965 \texttt{</answer>} \\
\textcolor{green!50!black}{$\Rightarrow$ $r_3 = +2.79$ (the answer turn raises $\tilde p_\theta$ to $0.810$ across both paraphrases and adds the terminal augmentation $(1-\lambda)q_\theta = 0.35$).}
\end{tcolorbox}
    \caption{Multi-reference SIOP case study on a multi-hop query. Two surface forms of the gold date coexist in the dominant cluster, so the surrogate $\tilde p_\theta$ averages prefix support across both paraphrases with uniform weights. The first search briefly misleads the prefix (red), the second retrieves the husband identity (small green), and the final answer assigns high support to both paraphrases (large green). Numbers come from a single optimization step of the full SIOP run.}
    \label{fig:case_study_multiref}
    \vspace{-1em}
\end{figure}

\section{Reward Hacking and Confident-but-Wrong Modes}
\label{app:reward_hacking}

In verifier-free optimization, self-generated targets can reward an answer mode simply because many rollouts converge to it. This creates a reward-hacking path for confident-but-wrong modes: semantic agreement may reflect repeated model bias rather than evidence-grounded correctness. SIOP addresses this by separating cluster frequency from reliability before either signal enters the reward. In the factorized structure of Proposition~\ref{prop:reduction}, the target distribution $q_\theta(c \mid q)$ is not a direct function of cluster frequency. It combines empirical mode mass $m(c \mid q)$ with a reliability correction $u_\theta(c,q)$ driven by evidence support. Proposition~\ref{prop:calibration_reversal} makes this behavior explicit: a lower-frequency but evidence-consistent cluster can receive larger target mass whenever its reliability advantage is large enough to overcome the log frequency gap. Thus raw majority support alone is not enough to dominate the target distribution.

The same separation carries into turn-level credit assignment. The outcome potential weights prefix support by the calibrated target mass of the rollout's semantic outcome mode, so turns that increase support for reliable modes receive stronger credit than turns that only make a high-frequency mode more predictable. Cluster-level averaging across multiple rollouts and references further prevents a single wording choice from dominating the target.

\section{Limitations and Future Work}
\label{app:limitations}

\noindent \textbf{Limitations.} The experiments focus on search-augmented QA with Qwen3-4B and Qwen3-8B, a setting where semantic answer modes and retrieved evidence provide a controlled testbed for turn-level credit assignment. This scope keeps the evaluation focused, while leaving broader agentic environments to future study. Since SIOP uses generic semantic clustering and evidence-support models, its signal is most natural when final outcomes can be grouped by meaning and retrieved evidence can support or contrast those meanings. Settings with highly ambiguous outcomes, weak evidence, or fine-grained tool states may require adapting the outcome-state construction.

\noindent \textbf{Future work.} Future work could extend semantic outcome potentials in several directions. First, the same abstraction may apply beyond search-QA to agentic settings where trajectories end in semantically diverse outcomes, such as tool use, web navigation, and interactive code repair. Second, outcome states need not be limited to text-answer clusters. They could be defined over structured tool traces, multi-modal observations, or hierarchical semantic modes, allowing the reward to track richer forms of future progress. Third, the potential-based view opens the door to adaptive credit granularity, where the training signal is assigned not only at the turn level, but also to retrieved evidence, tool actions, or finer action spans when they are the main source of progress.
\newpage

\end{document}